\documentclass[letterpaper]{article} 
\usepackage[draft]{aaai2026}  
\usepackage{times}  
\usepackage{helvet}  
\usepackage{courier}  
\usepackage[hyphens]{url}  
\usepackage{graphicx} 
\usepackage{natbib}  
\usepackage{caption} 
\usepackage{amsmath}
\usepackage{amssymb}
\usepackage{amsthm}  
\usepackage{ltl}
\usepackage{MnSymbol}
\usepackage{subcaption} 
\usepackage{tikz}

\usetikzlibrary{arrows,
	automata,
	calc,
	decorations,
	decorations.pathreplacing,
	positioning,
	shapes,
	calligraphy}
	\usetikzlibrary{bayesnet,  fit}

\tikzset{res/.style={ellipse,draw,minimum height=0.5cm,minimum width=0.8cm}}

\newcommand{\Fam}{\mathrm{Fam}}
\newcommand{\ind}{\upmodels}
\newcommand{\probind}{\perp}
\newcommand{\cA}{\mathcal{A}}

\newtheorem{theorem}{Theorem}[section]
\newtheorem{result}[theorem]{Main result}
\newtheorem{lemma}[theorem]{Lemma}
\newtheorem{proposition}[theorem]{Proposition}
\theoremstyle{definition}
\newtheorem{definition}[theorem]{Definition}
\newtheorem{example}[theorem]{Example}

\title{Temporal Properties of Conditional Independence in Dynamic Bayesian Networks}

\author{
Rajab Aghamov\textsuperscript{\rm 1}\thanks{Authors listed in alphabetical order},
Christel Baier\textsuperscript{\rm 1},
Jo\"el Ouaknine\textsuperscript{\rm 2},\\
Jakob Piribauer\textsuperscript{\rm 1},
Mihir Vahanwala\textsuperscript{\rm 2},
Isa Vialard\textsuperscript{\rm 2}
}

\affiliations{
\textsuperscript{\rm 1}Technische Universit\"at Dresden\\
\textsuperscript{\rm 2}Max Planck Institute for Software Systems, Saarland Informatics Campus\\
rajab.aghamov@mailbox.tu-dresden.de
}

\begin{document}

\maketitle

\begin{abstract}
Dynamic Bayesian networks (DBNs) are compact graphical representations used to model probabilistic systems where interdependent random variables and their distributions evolve over time. In this paper, we study the verification of the evolution of conditional-independence (CI) propositions against temporal logic specifications. To this end, we consider two specification formalisms over CI propositions: linear temporal logic (LTL), and non-deterministic Büchi automata (NBAs). This problem has two variants. Stochastic CI properties take the given concrete probability distributions into account, while structural CI properties are viewed purely in terms of the graphical structure of the DBN. We show that deciding if a stochastic CI proposition eventually holds is at least as hard as the Skolem problem for linear recurrence sequences,  a long-standing open problem in number theory. On the other hand, we show that verifying the evolution of structural CI propositions against LTL and NBA specifications is in PSPACE, and is  NP- and coNP-hard. We also identify natural restrictions on the graphical structure of  DBNs that make the verification of structural CI properties tractable.
\end{abstract}

\begin{links}
\link{Extended version}{this}
\end{links}

\section{Introduction}

Bayesian networks (BNs)  \cite{pearl1985bayesian,pearlbook,neapolitan1989probabilistic} are prominent tools in both data science and artificial intelligence that enable modeling and reasoning under uncertainty. BNs succinctly represent a full joint probability distribution
by using a directed acyclic graph (DAG) as a template to capture dependencies between variables and prescribe the probability distribution of each variable conditioned on its parents.
BNs have successfully been applied in medical AI \cite{lucas2004probabilistic}, natural language processing \cite{manning1999foundations}, robotics \cite{thrun2005probabilistic}, bioinformatics \cite{friedman2000using}, and risk assessment \cite{fenton2012risk}.

Dynamic Bayesian Networks (DBNs) extend BNs to describe systems where the outcomes modeled by random variables evolve with time \cite{murphyphd,koller2009probabilistic}. 
DBNs succinctly represent a \emph{sequence} of full joint probability distributions of a set of random variables, i.e., a DBN prescribes an initial joint probability distribution for the variables $\mathbf{V}^0$, and also prescribes the joint distribution of $\mathbf{V}^{t+1}$, the variables at time step (or time slice) $t+1$, conditioned on the variables $\mathbf{V}^t$ at time step $t$.
These are respectively given by an \emph{initial} BN and a \emph{step} BN, and their corresponding DAGs are collectively referred to as the \emph{DBN template}. To make a concrete DBN, the template is instantiated with \emph{conditional probability distributions} (CPDs).

The temporal dimension of DBNs has motivated applications in robotics \cite{thrun2005probabilistic} and systems biology \cite{dbnmodelcheck}.
Today, DBNs find applications in various areas, and the following examples illustrate the continued relevance of DBNs in modern AI.
Integrating DBNs can make computer-vision algorithms more adaptive and efficient \cite{piedade2023}, and furthermore, DBNs and LLMs have been combined to build multimodal AI systems that interact with users in a context-aware manner \cite{han2025}. In healthcare, DBNs support early sepsis prediction in the ICU while remaining interpretable and robust to missing data \cite{agard2025}. Recent neuroscience work uses multi-timescale DBNs to infer directed, behavior-dependent interactions between brain regions, demonstrating utility on high-quality datasets \cite{das2024}. Beyond medicine, DBNs are applied, e.g., to solar power generation forecasting \cite{zhang2024} and resilience analysis of dynamic engineering systems such as transportation networks \cite{kammouh2020}.
Given their relevance, algorithms for learning DBN structures from data are an active area of research (see, e.g.,  \citet{meng2024dbn}).

 \begin{figure*}[htbp]
    \centering
    \begin{subfigure}[b]{0.38\textwidth}
        \centering
        
         \resizebox{.8\textwidth}{!}{
        \begin{tikzpicture}[
    latent/.style={circle, draw=black, fill=white, inner sep=1pt, minimum size=2em},
    node distance=.8cm and .8cm, 
    >=stealth' 
]

    \node[latent] (W_t) {$L$};
    \node[latent, right=of W_t] (RT_t) {$O$}; 
    \node[latent, right=of RT_t] (Th_t) {$S$}; 
    \node[latent, right=of Th_t] (H_t) {$H$};    

    \node[above=of W_t,xshift=-2cm,yshift=-0.5cm] (anchor1) {};
    \node[above=of H_t,xshift=1cm,,yshift=-0.5cm] (anchor2) {};
    
    \draw [dashed] (anchor1) -- (anchor2);

    \node[latent, above=of W_t] (W_0) {$L^0$};
    \node[latent, right=of W_0] (RT_0) {$O^0$}; 
    \node[latent, right=of RT_0] (Th_0) {$S^0$}; 
    \node[latent, right=of Th_0] (H_0) {$H^0$};    

    \node[latent, below=of W_t] (W_t1) {$L'$}; 
    \node[latent, right=of W_t1] (RT_t1) {$O'$};
    \node[latent, right=of RT_t1] (Th_t1) {$S'$};
    \node[latent, right=of Th_t1] (H_t1) {$H'$};

    \edge {W_0} {RT_0};
    \edge {H_0} {Th_0};

    \edge {W_t} {W_t1};
    \edge {RT_t} {W_t1};
        \edge {W_t1} {RT_t1};
          \edge {H_t} {H_t1};
           \edge {Th_t} {H_t1};
      \edge {RT_t} {H_t1};
      \edge {H_t1} {Th_t1};

    \node[draw, rectangle, rounded corners, fit=(W_t) (RT_t) (Th_t) (H_t), label={[anchor= west,xshift=-30] west:$t$}] (plate_t) {};
    \node[draw, rectangle, rounded corners, fit=(W_t1) (RT_t1) (Th_t1) (H_t1), label={[anchor= west,xshift=-30] west:$t+1$}] (plate_t1) {};
     \node[draw, rectangle, rounded corners, fit=(W_0) (RT_0) (Th_0) (H_0), label={[anchor= west,xshift=-30] west:$0$}] (plate_0) {};

       \node[below=of W_t1,yshift=.5cm] {\textcolor{white}{$\vdots$}};
   
\end{tikzpicture}
}
        
        \caption{Example of a DBN-template.}
        \label{fig:sub1}
    \end{subfigure}
     \hspace{-36pt}
    \begin{subfigure}[b]{0.38\textwidth}
        \centering
       
       \resizebox{.8\textwidth}{!}{
       \begin{tikzpicture}[
    latent/.style={circle, draw=black, fill=white, inner sep=1pt, minimum size=2em},
    node distance=.8cm and .8cm, 
    >=stealth' 
]

    \node[latent] (W_t) {$L^1$};
    \node[latent, right=of W_t] (RT_t) {$O^1$}; 
    \node[latent, right=of RT_t] (Th_t) {$S^1$}; 
    \node[latent, right=of Th_t] (H_t) {$H^1$};    

    \node[above=of W_t,xshift=-2cm,yshift=-0.5cm] (anchor1) {};
    \node[above=of H_t,xshift=1cm,,yshift=-0.5cm] (anchor2) {};

    \node[latent, above=of W_t] (W_0) {$L^0$};
    \node[latent, right=of W_0] (RT_0) {$O^0$}; 
    \node[latent, right=of RT_0] (Th_0) {$S^0$}; 
    \node[latent, right=of Th_0] (H_0) {$H^0$};    

    \node[latent, below=of W_t] (W_t1) {$L^2$}; 
    \node[latent, right=of W_t1] (RT_t1) {$O^2$};
    \node[latent, right=of RT_t1] (Th_t1) {$S^2$};
    \node[latent, right=of Th_t1] (H_t1) {$H^2$};

    \node[below=of W_t1,yshift=.5cm] {$\vdots$};
      \node[below=of RT_t1,yshift=.5cm] {$\vdots$};
        \node[below=of Th_t1,yshift=.5cm] {$\vdots$};
          \node[below=of H_t1,yshift=.5cm] {$\vdots$};

    \edge {W_0} {RT_0};
      \edge {H_0} {Th_0};

 \edge {W_t} {RT_t};
      \edge {H_t} {Th_t};

 \edge {W_t1} {RT_t1};
      \edge {H_t1} {Th_t1};

    \edge {W_t} {W_t1};
    \edge {RT_t} {W_t1};
        \edge {W_t1} {RT_t1};
          \edge {H_t} {H_t1};
           \edge {Th_t} {H_t1};
      \edge {RT_t} {H_t1};
      \edge {H_t1} {Th_t1};

    \edge {W_0} {W_t};
    \edge {RT_0} {W_t};
        \edge {W_t} {RT_t};
          \edge {H_0} {H_t};
           \edge {Th_0} {H_t};
      \edge {RT_0} {H_t};
      \edge {H_t} {Th_t};

    \node[draw, rectangle, rounded corners, fit=(W_t) (RT_t) (Th_t) (H_t), label={[anchor= west,xshift=-30] west:$1$}] (plate_t) {};
    \node[draw, rectangle, rounded corners, fit=(W_t1) (RT_t1) (Th_t1) (H_t1), label={[anchor= west,xshift=-30] west:$2$}] (plate_t1) {};
     \node[draw, rectangle, rounded corners, fit=(W_0) (RT_0) (Th_0) (H_0), label={[anchor= west,xshift=-30] west:$0$}] (plate_0) {};
   
\end{tikzpicture}
}

        \caption{Unfolding of the DBN-template.}
        \label{fig:sub2}
    \end{subfigure}
     \begin{subfigure}[b]{0.18\textwidth}
     \begin{center}
\begin{tabular}{c c |c c}
&& \multicolumn{2}{c}{$L'$} \\
$L$ & $O$ & 0 & 1 \\
\hline
0 & 0 & 0.2 & 0.8 \\
0 & 1 & 0.5 & 0.5 \\
1 & 0 & 0.3 & 0.7 \\
1 & 1 & 0 & 1 \\
\end{tabular}

\end{center}
\caption{Example of a CPD.}
\label{fig:sub3}
     \end{subfigure}
 
    \caption{The DBN-template described in Ex.~\ref{ex:intro} and its unfolding as well as an example CPD.}

    \label{fig:intro_figure}
\end{figure*}
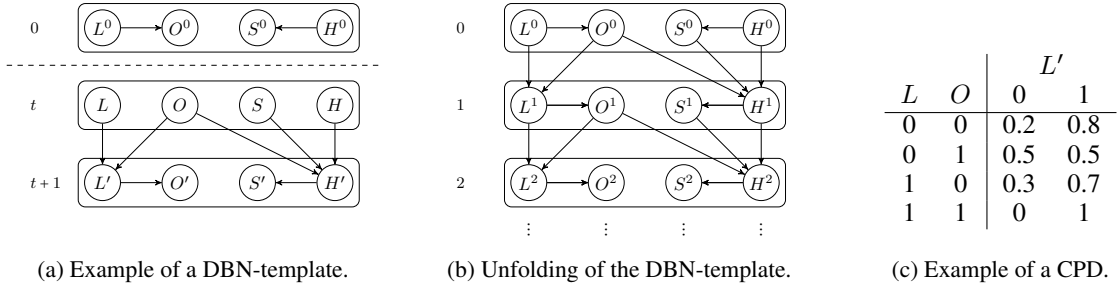

\begin{example}
\label{ex:intro}
To illustrate DBNs and DBN-templates, consider a  system coordinating different probabilistic components  either having access only to \emph{low-security} information
or also to \emph{high-security} information. In each time step $t$, the low-security components provide an input $L^t$ and the high-security components provide an input $H^t$.
The system then produces a low-security output $O^t$ and a high-security output $S^t$ ($S$ for secret).
The dependencies between these  variables are depicted in the 
 DBN-template depicted in Fig.~\ref{fig:sub1}: The \emph{initial template} marked with $0$ expresses that initially $O^0$ depends only on $L^0$ and $S^0$ depends only on $H^0$.
 All other pairs of variables are independent.
 The \emph{step template} depicted below uses the variables $L$, $O$, $S$, and $H$ representing the variables at the current time step as well as copies $L'$, $O'$, $S'$, and $H'$ representing 
 the variables at the next time step. For example, this template captures  that the next low-security input $L^{t+1}$ always depends directly only on the previous low-security input $L^t$ and output  $O^t$.
To represent all (direct) dependencies between variables in the timed sequence of variables, the template can be unfolded into one infinite directed acyclic graph (DAG), called the \emph{unfolding of the DBN-template},
as depicted in Fig.~\ref{fig:sub2}.

A DBN based on this DBN-template additionally consists of the CPDs  for each variable  in  the initial template and for each primed variable in the step template given its parents. Assuming that all variables 
take only values $0$ and $1$, an example CPD for the variable $L'$ in the step template is depicted in Fig.~\ref{fig:sub3}. For each combination of values of the 
parent variables $L$ and $O$, it specifies the probability with which $L'$ takes value $0$ and $1$. 
This CPD is  applied at each time step in the unfolding of the DBN.
For variables without parents, such as $L^0$ and $H^0$ in the initial template,
the CPDs  specify the probabilities with which they take value $0$ and $1$.
\end{example}

A fundamental concern (see, e.g., \citet{Hman1995}) in reasoning about probabilistic models is the characterization and/or deduction of stochastic conditional independence (CI) of sets $\mathbf{X}$ and $\mathbf{Y}$ of random variables given the values of a set $\mathbf{Z}$, denoted $(\mathbf{X} \probind \mathbf{Y} \mid \mathbf{Z})$. 

In seminal work, \cite{geiger1990} show that the truth of these CI propositions, which satisfy the so-called graphoid axioms, can be deduced from the underlying DAG template in the case of BNs.
Specifically, they define structural conditional independence through the efficiently testable graphical notion of \emph{d-separation}, denoted as $(\mathbf{X} \ind \mathbf{Y} \mid \mathbf{Z})$ when $\mathbf{Z}$ d-separates $\mathbf{X}$ and $\mathbf{Y}$.
They then show soundness: if $(\mathbf{X} \ind \mathbf{Y} \mid \mathbf{Z})$, then $(\mathbf{X} \probind \mathbf{Y} \mid \mathbf{Z})$ for all conditional probability distributions. Subsequently, \citet{Meek} showed a form of completeness: i.e., if $(\mathbf{X} \ind \mathbf{Y} \mid \mathbf{Z})$ fails then $(\mathbf{X} \probind \mathbf{Y} \mid \mathbf{Z})$ also fails for all but a (Lebesgue) measure-$0$ set of conditional probability distribution parameters.

DBNs can naturally be associated with an infinite sequence of BNs, with the $t$-th term being obtained by unfolding up to time slice $t$. We call a statement of the form $(\mathbf{X} \ind \mathbf{Y} \mid \mathbf{Z})$ (respectively, $(\mathbf{X} \probind \mathbf{Y} \mid \mathbf{Z})$) an atomic proposition of structural (respectively, stochastic) CI, and say that it holds at time $t$ if $(\mathbf{X}^t \ind \mathbf{Y}^t \mid \mathbf{Z}^t)$ (respectively, $(\mathbf{X}^t \probind \mathbf{Y}^t \mid \mathbf{Z}^t)$) holds in the unfolding of the DBN up to time $t$. Given a collection $A$ of structural (respectively, stochastic) CI propositions, the DBN defines a \emph{trace}, i.e., an infinite word over the alphabet $2^A$ whose $t$-th position records which of the propositions hold at time $t$.

In this paper, we concern ourselves with checking the properties of the trace, such as: is $(\mathbf{X} \ind \mathbf{Y} \mid \mathbf{Z})$ ever false? In Ex.~\ref{ex:intro}, is the output always independent of the secret given the low-security input? In a system, is it always the case that if inputs $I_1, I_2$ are independent, then so are outputs $O_1, O_2$?

To express temporal properties of systems, the use of temporal logics such as linear temporal logic (LTL) and of non-deterministic B\"uchi automata (NBAs), capturing all $\omega$-regular languages, has emerged as a success story over the past decades (see, e.g., \cite{baierkatoen}).
We aim to employ these formalisms to talk about the temporal aspects of CIs.

\begin{example}
\label{ex:temporal}
The three properties mentioned above are expressed in LTL as: (i) $\LTLeventually \neg(\mathbf{X} \ind \mathbf{Y} \mid \mathbf{Z})$, where $\LTLeventually$ is the temporal modality for ``eventually''; (ii) $\LTLglobally (O \ind S \mid L)$, where $\LTLglobally$ is the temporal modality for ``globally'', and dual to $\LTLeventually$; (iii) $\LTLglobally ((I_1 \ind I_2) \rightarrow (O_1 \ind O_2))$.
\end{example}

Given an LTL formula $\varphi$ over the set of atomic propositions $A$ or an NBA $\mathcal{B}$ over the alphabet $2^A$, 
the \emph{structural CI model-checking problem for DBN-templates} now asks whether the trace of a DBN-template
satisfies $\varphi$ or is accepted by $\mathcal{B}$, respectively.
The  \emph{stochastic CI model-checking problem for DBNs} asks the analogous question for the trace of a DBN with respect to a set of stochastic CI propositions. 

\subsection{Contributions}

\begin{enumerate}
\item
In Sec.~\ref{section:model-checking}, we introduce temporal specification mechanisms  
for the evolution of structural or stochastic CI propositions in DBN-templates and DBNs, respectively, using LTL and NBAs.
We formulate the resulting structural and stochastic CI model-checking problems.

\item
In Sec.~\ref{section:structural}, we show that the structural CI model-checking problems of DBN-templates against LTL formulas and against NBAs are both in PSPACE and NP-hard as well as coNP-hard.
Under the natural restriction that the initial template of a DBN-template only contains edges that also appear as intra-slice edges in the step template, we prove that the problems are in P.

\item
Given full DBNs with CPDs, we show in Sec.~\ref{section:skolem} that checking eventual stochastic CI is as hard as the Skolem problem for linear recurrence sequences, a famous
number-theoretic problem whose decidability status has been open for many decades.
This implies that a decidability result for the stochastic CI model-checking problems is out of reach without a breakthrough in analytic number theory.

\end{enumerate}

\subsection{Related Work}
The question how to detect structural CIs in BNs has been answered in the 1980s and 1990s by showing that d-separation characterizes all structural CIs that follow from the structure
of a BN, that this is equivalent to stochastic CI under almost all choices of CPDs, and by showing that the d-separation can compute these structural CIs in polynomial time (see \cite{pearlbook,geiger1990,Meek}).
Exactly determining whether a stochastic CI holds requires exact computation of the necessary conditional probability distributions.
Methods for approximate testing of conditional independence of discrete random variables, however, are an active area of research (see, e.g., 
\cite{DBLP:conf/stoc/CanonneDKS18,Teymur2020}).
Orthogonally, seminal work by \citet{Boutilier1996Context} studies so-called context-sensitive independence expressing that variables might only be independent under specific  assignments of values to other variables. Like stochastic CI, this kind of independence  depends on the concrete CPDs.

We are not aware of thorough studies of d-separation and the detection of CIs in DBNs, let alone the formal verification of temporal properties of CIs in DBNs.
Regarding other extension of BNs,
\citet{DBLP:conf/icml/ShenHCD19} study CIs in testing BNs, an extension of BNs representing a set of probability distributions instead of a singly distributions, 
and show that d-separation can still be used to detect structural CIs.

\section{Preliminaries}
	\paragraph{Probability spaces and conditional independence.}
	We assume knowledge of the basics of probability theory \cite{Klenketextbook}, and record the relevant prerequisites for completeness. A probability space is a measure space given by a triple $(\Omega, \mathcal{F}, P)$. Here, $\Omega$ is the sample space, i.e., a non-empty set of mutually exclusive and collectively exhaustive outcomes (e.g., the set $\{\text{HH, HT, TH, TT}\}$ of outcomes of two coin tosses). The set of events $\mathcal{F}$ is a $\sigma$-algebra, i.e., a set of subsets of $\Omega$, called \emph{events}, which includes $\Omega$, and is closed under complement, countable union, and hence also countable intersection by De Morgan's law (in our example we can define $\mathcal{F}$ to have $16$ events). Finally, $P: \mathcal{F} \rightarrow [0, 1]$ is the probability measure which satisfies $P(\Omega) = 1$ and $\sigma$-additivity, i.e., if $\{A_i\}_i$ is a countable collection of pairwise disjoint sets, then $P\left(\bigcup_i A_i \right) = \sum_i P(A_i)$ (e.g., $P(\{\}) = 0, P(\{\text{HH}\}) = 1/4, P(\{\text{HT}\}) = 1/4, P(\{\text{TH}\}) = 1/3, P(\{\text{TT}\}) = 1/6$).
	
	A random variable $X$ on a probability space $(\Omega, \mathcal{F}, P)$ is a measurable function from $\Omega$ to a measurable space $S$\footnote{We equip $S$ with $\sigma$-algebra $\mathcal{F}'$ and measure $\mu: \mathcal{F}' \rightarrow \mathbb{R}$.}, called its \emph{support}, i.e., for any measurable $T \subseteq S$, $X^{-1}(T) \in \mathcal{F}$. We call $X$ discrete if $S$ is countable, and binary if $S = \{0, 1\}$. In this paper, we shall work with discrete random variables, which will often be binary for simplicity. Given $x \in S$, we write $p_X(x) = \Pr[X = x] = P(\{\omega \in \Omega : X(\omega) = x\})$. For example, for the random variable $X$ which indicates whether the second coin shows Heads, $p_X(1) = 7/12$. We analogously define $p_\mathbf{X}: \mathbf{S} \rightarrow [0, 1]$ for a finite tuple $\mathbf{X}$ of discrete random variables with support $\mathbf{S}$.

In this paper, we work with discrete random variables. Disjoint tuples of random variables $\mathbf{X}, \mathbf{Y}$ are considered conditionally independent given $\mathbf{Z}$ (denoted as $(\mathbf{X} \probind \mathbf{Y} \mid \mathbf{Z})$), if for any values $\mathbf{x}, \mathbf{y}, \mathbf{z}$ (provided $\Pr[\mathbf{Z} = \mathbf{z})] > 0$, the following holds: $\Pr(\mathbf{X} = \mathbf{x}, \mathbf{Y} = \mathbf{y}\mid \mathbf{Z}=\mathbf{z}) = \Pr(\mathbf{X} = \mathbf{x}\mid \mathbf{Z} = \mathbf{z}) \cdot \Pr(\mathbf{Y} = \mathbf{y}\mid \mathbf{Z} = \mathbf{z})$.

\paragraph{Bayesian networks}
A Bayesian network (BN) is a type of probabilistic graphical model that expresses a set of variables and their conditional dependencies using a directed acyclic graph (DAG), where each node represents a variable, and the edges indicate direct probabilistic dependencies between the variables.

\begin{definition}[Bayesian Network]
	A Bayesian network (BN) over a finite set $\mathbf{V}$ of discrete random variables is a tuple $B = \langle \mathbf{V}, \mathcal{E}, \mathcal{P} \rangle$, where:
	\begin{itemize}
		\item Each element of $\mathbf{V}$ is represented as a vertex of a DAG;
		\item The set of directed edges is $\mathcal{E} \subseteq \mathbf{V} \times \mathbf{V}$;
		we call the DAG $\langle \mathbf{V}, \mathcal{E}\rangle$ the \emph{template} of the BN.
		\item The probability distribution of $\mathbf{V}$ is expressed in terms of a collection $\mathcal{P}$ of conditional probability distributions (CPDs), i.e., for each variable $X$ with parents $ \mathsf{pa}(X) = \mathbf{Y}$, $\mathcal{P}$ prescribes $\Pr[X = x \mid \mathbf{Y} = \mathbf{y}]$ for all possible $x, \mathbf{y}$.
\end{itemize}
We refer to the set of all BNs with a given BN-template $\mathcal{T}$ as the family $\mathsf{Fam}(\mathcal{T})$.
\end{definition}

As an illustrative scenario, consider a BN where all variables are binary. Then $\mathcal{P}$ consists of $\sum_{X \in \mathbf{V}} 2^{|\mathsf{pa}(X)|}$ parameters, each term of the summation counting the number of parameters required to prescribe the probability of $X$ being $1$, depending on the values taken by its parents.

\begin{definition}[d-paths and d-separation]
Given a BN-template (i.e., DAG) $\mathcal{T}=(\mathbf{V},\mathcal{E})$ and three pairwise disjoint sets $\mathbf{X}, \mathbf{Y}, \mathbf{Z} \subseteq \mathbf{V}$ of nodes, a d-path from $\mathbf{X}$ to $\mathbf{Y}$ with respect to $\mathbf{Z}$ is a sequence $W_0, \ldots, W_k$ of nodes with the following properties:
\begin{itemize}
\item $W_0\in \mathbf{X}$ and $W_k\in \mathbf{Y}$,
\item for each $i<k$, either $(W_i, W_{i+1}) \in \mathcal{E}$ or $(W_{i+1}, W_i) \in \mathcal{E}$
\item  for all $0<i<k$, if the node $W_i$ has an outgoing edge to $W_{i-1}$ or $W_{i+1}$ (or both), then $W_i\not\in \mathbf{Z}$,
\item  for all $0<i<k$, if the node $W_i$ has incoming edges from both $W_{i-1}$ and $W_{i+1}$, then one of the descendants of $W_i$ is in $\mathbf{Z}$ (we consider a node to be its own descendant and ancestor).
We call such a node $W_i$ a \emph{collider} and say that the collision is \emph{attributed} to the descendants of $W_i$ in $\mathbf{Z}$.
\end{itemize}
If there is no such path, we say that $\mathbf{Z}$ d-separates $\mathbf{X}$ and $\mathbf{Y}$.  In this case, we write $(\mathbf{X} \ind \mathbf{Y} \mid \mathbf{Z})$ and say that $\mathbf{X}$ and $\mathbf{Y}$ are \emph{structurally independent} given $\mathbf{Z}$.
\end{definition}

We remark that both structural and stochastic conditional independence (CI) satisfy the \emph{graphoid} axioms \cite[p.\ 511, (4a)-(4d)]{geiger1990} (see also \cite{spohn1980} for a proof of the stochastic case). Structural CI via d-separation in a BN-template $\mathcal{T}$ is known to be equivalent to stochastic CI in all members of the family $\mathsf{Fam}(\mathcal{T})$. 

\begin{theorem}[Soundness and completeness of d-separation, \cite{pearlbook}, \cite{Meek}]
\label{thm:BN-equivalence}
Given a BN-template $\mathcal{T}=\langle \mathbf{V}, \mathcal{E} \rangle$ and pairwise disjoint sets of random variables $\mathbf{X}, \mathbf{Y}, \mathbf{Z} \subseteq \mathbf{V}$, the following two statements are equivalent and can be checked in polynomial time:
\begin{itemize}
\item $(\mathbf{X} \ind \mathbf{Y} \mid \mathbf{Z})$,
\item for all BNs in $\mathsf{Fam}(\mathcal{T})$, we have $(\mathbf{X} \probind \mathbf{Y} \mid \mathbf{Z})$, i.e., $\mathbf{X}$ and $\mathbf{Y}$ are (stochastically) independent given $\mathbf{Z}$.
\end{itemize}
\end{theorem}

\paragraph{Dynamic Bayesian networks (DBNs).}
We use Dynamic Bayesian networks \cite{murphyphd,koller2009probabilistic} to model probabilistic systems where the entities captured by random variables evolve over time. Formally, consider a finite set $\mathbf{V}$ of random variables. We track the evolution of $\mathbf{V}$ with time through a countably infinite sequence $(\mathbf{V}^t)_{t=0}^\infty$ of copies of the random variables in $\mathbf{V}$.

The evolution itself is modeled as a dynamical system whose state at time $t$ is a BN involving the variables $\bigcup_{i=0}^t \mathbf{V}^t$. The initial BN is given by $\langle \mathbf{V}^0, \mathcal{E}^0, \mathcal{P}^0 \rangle$.

We use a copy $\mathbf{V}'$ of the random variables to express the update dynamics, which at each step $t \ge 1$, introduce the variables $\mathbf{V}^t$ with parents in $\mathbf{V}^{t-1} \cup \mathbf{V}^{t}$, while keeping the rest of the network unchanged. Formally, we have a \emph{step template} which is a BN-template with variables $\mathbf{V} \cup \mathbf{V}'$ and edge relation $\mathcal{E}^\mathsf{step} \subseteq (\mathbf{V} \cup \mathbf{V}') \times \mathbf{V}'$. Semantically, if $X'$ has parents $Y_{i_1}', \ldots, Y_{i_\ell}', Y_{j_1}, \ldots, Y_{j_k}$ in the step template, then for each $t \ge 1$, $X^t$ has parents $Y_{i_1}^t, \ldots, Y_{i_\ell}^t, Y_{j_1}^{t-1}, \ldots, Y_{j_k}^{t-1}$. Finally, we have the conditional probability distribution parameters $\mathcal{P}^\mathsf{step}$ which prescribe $\Pr[X' = x \mid \mathsf{pa}(X') = \mathbf{y}]$ for all $x, \mathbf{y}$. Semantically, we have that the distribution of $X^t$ conditioned on its parents is the same for all $t \ge 1$.

We can thus specify a DBN via $\langle \mathbf{V}, \mathcal{E}^0, \mathcal{P}^0, \mathcal{E}^\mathsf{step}, \mathcal{P}^\mathsf{step} \rangle$. Its structural properties are given by the tuple $\mathcal{T}_\mathsf{DBN} = \langle \mathbf{V}, \mathcal{E}^0,  \mathcal{E}^\mathsf{step} \rangle$, which we refer to as a \emph{DBN-template}. Analogous to BNs, we refer to the set of all DBNs with a given template $\mathcal{T}_\mathsf{DBN}$ as the family $\mathsf{Fam}(\mathcal{T}_\mathsf{DBN})$.
As explained earlier, in a DBN of binary variables, $\mathcal{P}^0$ consists of $\sum_{X \in \mathbf{V}^0} 2^{|\mathsf{pa}(X)|}$ parameters, and $\mathcal{P}^\mathsf{step}$ consists of $\sum_{X \in \mathbf{V}'} 2^{|\mathsf{pa}(X)|}$ parameters.

We say that a DBN-template is \emph{restricted} if whenever $(X^0, Y^0) \in \mathcal{E}^0$, we also have that $(X', Y') \in \mathcal{E}^\mathsf{step}$. 
So, the dependencies that exist initially at time step $0$ also have to be present at all later time steps, which is captured by their presence in the step template between the corresponding primed variables. The DBN-template depicted in Fig.~\ref{fig:intro_figure} is an example of a restricted DBN template.

\paragraph{Equivalence of DBNs with Markov Chains.}
That the semantics of a DBN can be expressed in terms of the evolution of a Markov chain is folklore. For completeness, we state an elementary lemma expressing this fact formally.
\begin{lemma}
\label{lemma:unfolding}
Given a DBN with $k$ binary random variables, we can construct an equivalent Markov chain with $2^k$ states. 

Given a Markov chain with $K$ states, we can construct an equivalent DBN with $\lceil \log K \rceil$ binary random variables.
\end{lemma}
\begin{proof}
To show the first claim, we describe how to obtain a Markov chain from a DBN:
Each of the states of the Markov chain (labelled $0, 1, \ldots, 2^{k-1}$) indicates, through the binary expansion of the label, one of the $2^k$ possible configurations of the random variables. E.g., the state $1$ indicates all variables are $0$, except $X_0 = 1$. 

The conditional probabilities in the step-template network enable us to compute the update matrix $M$ of the Markov chain, column by column (our linear algebraic convention takes the matrix to be column-stochastic, i.e., the columns are distributions and sum up to $1$). The $(\ell, m)$-th entry is the probability of the configuration being $\ell$, given the previous configuration was $m$.

For the second claim, we show how to encode a Markov chain in a DBN:
We assume, without loss of generality, that the Markov chain has $K = 2^{k+1}$ states, labelled $0, 1, \ldots, 2^k - 1$. This can be done by adding extra states whose only transitions are self-loops. Our main idea is to perform the previous construction ``in reverse.''

To that end, we name the required random variables $X_0, \ldots, X_k$, and interpret their configuration at the current time slice as indicating the state the Markov chain is in. In the step-template network, each $X'_i$ depends upon $X_0, \ldots, X_k, X'_k, \ldots, X'_{i+1}$.  The idea is the same as that of a binary search: given the previous state of the Markov chain was $m$ (given by $X_k, \ldots, X_0$), we first consider the conditional distribution of the most significant bit $X_k'$ of the current state, and then that of the $(k-1)$-st bit $X_{k-1}'$ given $X_k, \ldots, X_0, X_k'$, and so on.
\end{proof}

\section{Specification Formalisms for Temporal Conditional Independence Properties}
\label{section:model-checking}

In this paper, we study the verification of DBNs
and DBN-templates against linear-temporal properties regarding the evolution of conditional independencies (CIs) over time.
For DBN-templates 
 with variables $\mathbf{V}$ we use atomic propositions of the form $(\mathbf{X} \ind \mathbf{Y} \mid \mathbf{Z})$, which we call \emph{structural CI propositions}.
 We say that $(\mathbf{X} \ind \mathbf{Y} \mid \mathbf{Z})$ holds in DBN-template $\mathcal{T}$ at time $t$ if $(\mathbf{X}^t \ind \mathbf{Y}^t \mid \mathbf{Z}^t)$ holds in the unfolding of $\mathcal{T}$  to a BN-template after $t$ time steps.
 Similarly, for full DBNs with concrete CPDs, we use \emph{stochastic CI propositions}  $(\mathbf{X} \probind \mathbf{Y} \mid \mathbf{Z})$ that hold at
 time $t$ in a DBN $\mathcal{B}$ if  $(\mathbf{X}^t \probind \mathbf{Y}^t \mid \mathbf{Z}^t)$ holds  in the BN formed at time $t$. 
A first result connecting structural CI and stochastic CI is immediate.
\begin{proposition}
	\label{sound&complete}
	If  $(\mathbf{X}^t\ind \mathbf{Y}^t \mid \mathbf{Z}^t)$ in a DBN-template $\mathcal{T}$ for some $t$, then,  $(\mathbf{X}^t\probind \mathbf{Y}^t\mid \mathbf{Z}^t)$ in every DBN $\mathcal{B}\in \Fam(\mathcal{T})$.
\end{proposition}

\begin{proof}
Suppose that $(\mathbf{X}^{t} \ind \mathbf{Y}^{t} \mid \mathbf{Z}^{t})$.  Let $B \in \Fam(\mathcal{T})$ be an arbitrary DBN, and let $B^{0:t}$ denote its unfolding into a BN over time slices $0$ to $t$. 
Note that $(\mathbf{X}^t\ind \mathbf{Y}^t \mid \mathbf{Z}^t)$ holds iff  there is no d-path wrt $\mathbf{Z}^t$ from $\mathbf{X}^t$ to $\mathbf{Y}^t$ in the BN-template $\mathcal{T}^{0:t}$.
By Thm.~\ref{thm:BN-equivalence}, it follows that  $(\mathbf{X}^t \probind \mathbf{Y}^t \mid \mathbf{Z}^t)$.	
\end{proof}

The converse direction, i.e., the completeness of d-separation for DBNs, however, turns out to be intricate. We discuss this issue in Sec.~\ref{section:discussion}.
 
 \paragraph{Trace of a DBN(-template).}
For any (finite) set $A$ of such structural CI propositions, a DBN-template defines a \emph{trace} $\tau \in (2^A)^\omega$, i.e., an infinite word over the alphabet $2^A$, where the letter at position $t$ indicates which propositions hold at time $t$. 
Likewise, a full DBN defines a trace $\tau \in (2^B)^\omega$ for any (finite) set $B$ of stochastic CI propositions.

We seek to verify whether the trace $\tau$ of a DBN-template or of a DBN satisfies a logical specification. We use the notation $\tau(t)$, to refer to the $t$-th position of $\tau$, and the notation $\tau[t: \infty]$ to refer to the suffix of $\tau$ starting at position $t$, e.g., $\tau[0:\infty] = \tau$, $\tau[t: \infty][t': \infty] = \tau[t+t': \infty]$.

\paragraph{Linear temporal logic (LTL).}
We consider two common logical formalisms (for a comprehensive exposition of which the reader is referred to \cite{baierkatoen}). The first is linear temporal logic (LTL, introduced in \cite{pnueliLTL}), whose formulae $\varphi$ over a set of atomic propositions $A$ are syntactically given by the grammar $\varphi := a \mid \neg \varphi \mid \varphi \land \varphi \mid \LTLnext \varphi \mid \varphi \LTLuntil \varphi$ where $a\in A$ is an atomic proposition. 
The operator $\LTLnext$ is called ``next'' and the operator $\LTLuntil$ is called ``until''.
Semantically, it is defined recursively whether an infinite word $\tau$ over $2^A$ satisfies an LTL formula (written $\tau \models \varphi$) as follows:
\begin{itemize}
\item $\tau \models a$ if and only if $a \in \tau(0)$,
\item $\tau \models \neg \varphi$ if and only if $\tau \not\models \varphi$,
\item $\tau \models \varphi_1 \land \varphi_2$ if and only if both $\tau \models \varphi_1$ and $\tau \models \varphi_2$,
\item $\tau \models \LTLnext \varphi$ if and only if $\tau[1: \infty] \models \varphi$,
\item $\tau \models \varphi_1 \LTLuntil \varphi_2$ if and only if there exists $t$ s.t. $\tau[t:\infty] \models \varphi_2$ and for all $t' < t$, $\tau[t':\infty] \models \varphi_1$.
\end{itemize}

For notational convenience, we allow access to all the usual Boolean connectives, true, false, as well as the temporal modalities $\LTLeventually$ (``eventually''; $\LTLeventually \varphi$ is equivalent to $\text{true } \LTLuntil \varphi$) and its dual $\LTLglobally$ (``globally''; $\LTLglobally \varphi$ is equivalent to $\neg \LTLeventually \neg \varphi$).

E.g., consider the DBN-template given in Example \ref{ex:intro}. 
We can use d-separation to argue that the structural formula $\LTLglobally (O \ind S \mid L)$ holds, i.e., $(O^t \ind S^t \mid L^t)$ holds for all $t$. 
Using the temporal LTL-operators, also more involved properties can be expressed:  the formula $(\mathbf{X}\ind \mathbf{Y}) \LTLuntil \neg(\mathbf{Y} \ind \mathbf{Z})$, e.g., expresses that 
the sets of variables $\mathbf{X}$ and $\mathbf{Y}$ are structurally independent at least until $\mathbf{Y}$ and $\mathbf{Z}$ are  dependent.

For LTL, we investigate the following two problems:

\begin{itemize}
\item
\textbf{Structural LTL model-checking of DBN-templates:} 
For a DBN-template $\mathcal{T}$ with variables $\mathbf{V}$ and an LTL formula $\varphi$ over the set $A$ of structural CI propositions using $\mathbf{V}$,
decide whether the trace $\tau\in (2^A)^\omega$ of $\mathcal{T}$ satisfies $\varphi$.

\item
\textbf{Stochastic LTL model-checking of DBNs:} 
For a DBN  $\mathcal{B}$ with variables $\mathbf{V}$ and an LTL formula $\varphi$ over the set $B$ of stochastic CI propositions using $\mathbf{V}$,
decide whether the trace $\tau\in (2^B)^\omega$ of $\mathcal{B}$ satisfies $\varphi$.
\end{itemize}

We indeed employ d-separation as a tool to reason more generally about specifications involving only structural independence propositions in Sec.~\ref{section:structural}. On the other hand, we prove that evaluating stochastic formulae as simple as $\LTLeventually (X \probind Y)$ can be number-theoretically hard (Lem.~\ref{lem-skolem-hardness}) showing that a decidability result for stochastic LTL or NBA model-checking of DBNs is  out of reach without a breakthrough in number theory.

\paragraph{Non-deterministic Büchi automata (NBAs).}
The second formalism we  consider is that of nondeterministic Büchi automata (NBAs, introduced in \cite{buchi}), which express precisely the class of $\omega$-regular temporal properties. An NBA is a tuple
$\cA = (Q,\Sigma,\Delta, Q_0, F)$ where
$Q$ is a finite set of states, $\Sigma$ is the alphabet, $\Delta\subseteq Q\times \Sigma \times Q$ is the transition relation,
$Q_0\subseteq Q$ is the set of initial states and $F\subseteq Q$ is the set of accepting states.
A run of $\cA$ on an infinite word $\tau = w_0w_1w_2\dots \in \Sigma^\omega$ is a sequence $\rho = q_0q_1\dots$ of states such that
$q_0\in Q_0$ and $(q_i,w_i,q_{i+1})\in \Delta$ for each $i\in \mathbb{N}$. The run $\rho$ is accepting if $q_j\in F$ for infinitely many $j\in \mathbb{N}$. We say that $\cA$ accepts $\tau$ if there exists an accepting run on $\tau$.

We remark that it is known that NBAs are strictly more expressive than LTL formulae, however translating an LTL formula into an equivalent NBA can lead to an exponential increase in size.
The resulting NBA model-checking problems we consider are:
\begin{itemize}
\item
\textbf{Structural NBA model-checking of DBN-templates:} 
For a DBN-template $\mathcal{T}$ with variables $\mathbf{V}$ and an NBA $\mathcal{A}$ over the alphabet $2^A$ where $A$ is the set of structural CI propositions using $\mathbf{V}$,
decide whether the trace $\tau\in (2^A)^\omega$ of $\mathcal{T}$ is accepted by  $\mathcal{A}$.
\item 
\textbf{Stochastic NBA model-checking of DBNs:} 
For a DBN $\mathcal{B}$ with variables $\mathbf{V}$ and an NBA $\mathcal{A}$ over the alphabet $2^B$ where $B$ is the set of stochastic CI propositions over $\mathbf{V}$,
decide if the trace $\tau\in (2^B)^\omega$ of $\mathcal{B}$ is accepted by  $\mathcal{A}$.

\end{itemize}

\section{Structural Conditional Independence}
\label{section:structural}

In this section, we  study the properties of the trace $\tau \in (2^A)^\omega$ of a DBN-template with variables $\mathbf{V}$ with respect to a set $A$ of structural CI propositions and show how to check whether the trace satisfies a logical specification. 
The main result we will establish is the following:

 \begin{result}
 The structural LTL and NBA model-checking problems for DBN-templates are in PSPACE and NP-hard as well as coNP-hard.
 
For restricted DBN-templates, the structural LTL and NBA model-checking problems are in PTIME.
 \end{result}

To prove this result, we will show that the trace of a DBN-template with respect to structural CI propositions is ultimately periodic by virtue of being represented as the run of a deterministic transition system with $2^{O(|\mathbf{V}|^2)}$ states. We shall further demonstrate that in this transition system, a state can be represented in $O(|\mathbf{V}|^2)$ space, the successor can be computed in time polynomial in $\mathbf{|V|}$, and the labeling function (i.e., which propositions hold in a given state) can be computed in time polynomial in $|\mathbf{V}|$, $|A|$.

We then rely on a classical argument whose key ingredients can be found in \cite[Proof of Lem.~5.47]{baierkatoen} and \cite[Thm.~3.2]{vardi1986automata}.
\begin{lemma}
Let $TS$ be a deterministic labelled transition system of size $N$, represented such that a state's label and successor can be computed in space polynomial in $\log N$, and let $\tau$ be its trace. Given a LTL formula $\varphi$ of size $M$, we can decide in space polynomial in $M, \log N$ whether $\tau$ satisfies $\varphi$. Given an NBA $\mathcal{A}$ of size $K$, we can decide in space polynomial in $\log K, \log N$ whether $\mathcal{A}$ accepts $\tau$.
\end{lemma}
\begin{proof}[Proof Sketch]
In the case of the LTL formula, the claim holds even if $TS$ is non-deterministic, see, e.g., \cite[Proof of Lem.~5.47]{baierkatoen} and \cite[Thm.~3.2]{vardi1986automata}. In the case of the NBA, we consider the synchronized product of $TS$ with $\mathcal{A}$. The runs of this transition system are all runs of $\mathcal{A}$ on $\tau$, and we use the same idea as that in \cite[Proof of Lem.~5.47]{baierkatoen}, i.e., non-deterministically guess (on-the-fly) a polynomially long prefix of an accepting run, and use only logarithmic space to count the length of the guessed sequence of states, verify that it is locally coherent to indeed be a run, and flag that the loop in the prefix visits an accepting state. Finally, recall that by Savitch's theorem, deterministic and non-deterministic polynomial space coincide.  
\end{proof}

 The improved complexity for restricted DBN-templates follows from the insight that  the trace of a restricted DBN-template on variables $\mathbf{V}$ is constant from time  $|\mathbf{V}|^2$
 onwards.
 
To convey a feeling for the hardness results, we first provide an example of a family of DBN-templates whose traces have periods exponential in the number of used variables: 

\begin{figure}
	\centering
	\scalebox{1}{
\begin{tikzpicture}[scale = 1.7, thick, >=latex]

	\def\xx{0}
	\def\yy{0}
	\def\zz{.5}

	\node[fill=black, circle, inner sep=1.5pt, label={above:$X_{i-1}$}] (a01) at (-.5,2) {};
	\node[fill=black, circle, inner sep=1.5pt, label={above:$W_{i,0}$}] (z0n) at (0,2) {};
	\node[fill=black, circle, inner sep=1.5pt, label={above:$W_{i,1}$}] (z02) at (.5,2) {};
	\node[fill=black, circle, inner sep=1.5pt, label={above:$W_{i,p_i-1}$}] (z01) at (2.5,2) {};

		\node[fill=black, circle, inner sep=1.5pt, label={above:$X_i$}] (a02) at (3,2) {};

	\node[fill=black, circle, inner sep=1.5pt, label={below:$X_{i-1}'$}] (a11) at (-.5,1) {};
	\node[fill=black, circle, inner sep=1.5pt, label={below:$W_{i,0}'$}] (z1n) at (0,1) {};
	\node[fill=black, circle, inner sep=1.5pt, label={below:$W_{i,1}'$}] (z12) at (.5,1) {};
	\node[fill=black, circle, inner sep=1.5pt, label={below:$W_{i,p_i-1}'$}] (z11) at (2.5,1) {};

	\node[fill=black, circle, inner sep=1.5pt, label={below:$X'_i$}] (a12) at (3,1) {};

	\node at (1,2) {$\textbf{\dots}$};
	\node at (1.5,0.9) {$\textbf{\dots}$};

	\draw[->] (z0n) -- (z12);
	\draw[->] (z02) -- (1,1);
	\draw[->] (1.5,2) -- (2,1);
	\draw[->] (2,2) -- (z11);
	\draw[->] (z01) -- (z1n);
	
	\draw[->,dashed] (a01) -- (z0n);

	\draw[ ->,dashed] (a02) to[out=165, in=-35] (z0n);

\end{tikzpicture}
}
    
\caption{Bridge between islands $X_{i-1}$ and $X_i$. Initial template in dashed edges.}
     
\label{fig:dbnBig}

\end{figure}
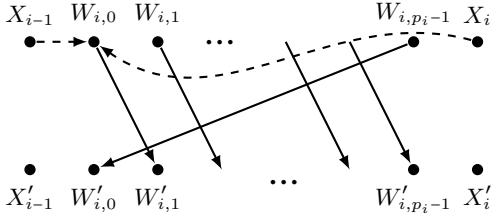

\begin{example}
The $k$-th DBN-template of this family has variables $X_0, X_1, \ldots, X_{k}$ as well as $W_{1,0}, W_{1,1},W_{2,0}, \ldots, W_{k,0}, W_{k,1}, \ldots, W_{k, p_k-1}$, where $p_i$ refers to the $i$-th prime. The initial template has edges of the form $(X_{i-1}^0, W_{i, 0}^0)$ and $(X_{i}^0, W_{i, 0}^0)$ for all $i$, and the step template has edges $(X_0, X_0'), (X_k, X_k')$, and edges from $W_{i, r}$ to $W'_{i, r+1 (\text{mod } p_i)}$. By construction, we have $(X_0 \ind X_k \mid \{W_{1,0}, \ldots W_{k,0}\})$ if and only if the timestep $t$ is divisible by all of $2, 3, \ldots, p_k$ (see Fig.~\ref{fig:dbnBig}).
Intuitively, we go from $X_0$ to $X_k$ via ``islands'' $X_1, \ldots, X_{k-1}$. The ``bridge'' between successive islands $X_{i-1}, X_i$ uses the edges of the initial template, and is open precisely when $t$ is divisible by $p_i$, i.e., a collision can be attributed to $W_{i,0}^t$. We need all bridges to be open simultaneously to make the journey: this gives us a period $2 \cdot 3 \cdot \cdots \cdot p_k$, which is exponential in $|\mathbf{V}| = (k+1)+2+3+\cdots+p_k$.
\end{example}

We establish NP-hardness of deciding whether $\LTLeventually (X \ind Y \mid \mathbf{Z})$ holds: we reduce from the NP-complete intersection problem for unary DFA \cite{blondiNP}. Since we also have access to the negated formula, we can also reduce from the complementary problem. Further, both properties can be expressed by fixed NBAs, and we get the following result whose proof is an adaptation of the example above.
\begin{lemma}
\label{np-conp-hard}
The LTL and NBA model-checking problems with structural-independence propositions are hard for NP as well as for coNP.
\end{lemma}
\begin{proof}
Assume we are given unary DFAs $\mathcal{A}_1, \ldots, \mathcal{A}_k$. Recall that each $\mathcal{A}_i$ is given by a set $\mathbf{Q}_i = \{Q_{i,0}, \ldots, Q_{i,m_i}\}$ of states, unary alphabet $\{a\}$, the initial state $Q_{i,0}$, transition function $\delta_i: \mathbf{Q}_i \rightarrow \mathbf{Q}_i$, and set of accepting states $\mathbf{F}_i \subseteq \mathbf{Q}_i$. The intersection problem asks whether there exists $t \ge 0$ such that $a^t$ is accepted by all the given automata. We encode this as a DBN with variables $\{X_0, \ldots, X_k\} \cup \mathbf{Q}_1 \cdots \cup \mathbf{Q}_k$. Let $\mathbf{F} = \mathbf{F}_1 \cup \cdots \cup \mathbf{F}_k$. In the initial template, we build bridges $(X_{i-1}^0, Q_{i,0}^0), (X_i^0, Q_{i,0}^0)$. In the step template, we connect $(X_0, X_0'), (X_k, X_k')$, and $(Q_{i,j}, \delta_i(Q_{i, j})')$. 

We then observe that $a^t$ is accepted by all automata if and only if $(X_0 \ind X_k \mid \mathbf{F})$ holds at time $t$. To see this, note that a d-path from $X_0^t$ to $X_k^t$ first has 
to go up to $X_0^0$.  Then, it has to move to $Q_{1,0}^0$ from there, it can only continue to $X_1^0$ if $Q_{1,0}^0$ (or one of its decendants) can serve as a collider.
This is only possible if an element of $\mathbf{F}^t$ is reachable from $Q_{1,0}^0$, which is the case if and only if the word $a^t$ leads to an accepting state in $\mathcal{A}_1$.
Analogously, $a^t$ also has to be accepted by all other given unary DFAs. 
\end{proof}

\subsection{DBN traces through transition systems}
We shall now prove the PSPACE upper bounds for the structural model-checking problems by showing that any trace of a DBN-template with variables $\mathbf{V}$ and structural independence propositions $A$ can be obtained as the run of a deterministic transition system with $2^{O(|\mathbf{V}|^2)}$ states, each of whose states can be represented in $O(|\mathbf{V}|^2)$ space, whose successor function can be computed in time polynomial in $|\mathbf{V}|$, and whose labelling function can be computed in time polynomial in $|\mathbf{V}|, |A|$.

We start by making a key observation about d-paths: 

\begin{lemma}
If there exists a d-path from $X$ to $Y$ with respect to a set $\mathbf{Z}$ of size $k$  in an arbitrary BN, there is one with at most $k$ collisions, all of which occur inside $\mathbf{Z}$. 
\end{lemma}
\begin{proof}
Suppose, for the sake of contradiction, a d-path with the fewest collisions has more than $k$ collisions. Then, by the pigeonhole principle there exists $Z \in \mathbf{Z}$ to which more than one collision is attributed. Let the first and last of these collisions occur at $W_1, W_2$ respectively. Since they have $Z$ as a common descendant, there exists a path from $W_1$ to $W_2$ with a single collision inside $\mathbf{Z}$. We use this path to ``tunnel through'' and replace the given d-path with one that has fewer collisions: a contradiction, as desired.
\end{proof}

The above claim motivates us to compute the pairs in $\{X^t, Y^t, Z_1^t, \ldots, Z_k^t\}$ that are connected by collision-free paths in order to determine whether there is a d-path relative to $\mathbf{Z}^t$ from $X^t$ to $Y^t$. However, we must be slightly careful: a path that concatenates collision-free paths via edges $(Z^t, X^t)$ and $(Z^t, Y^t)$ is actually blocked by $Z^t$.

\paragraph{Construction of the transition system.}

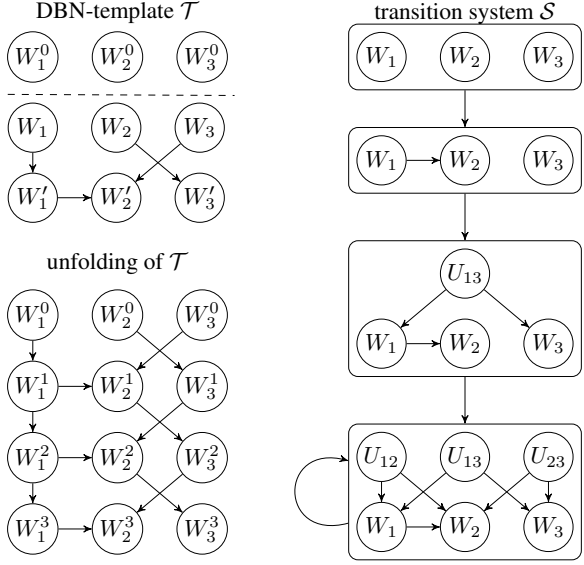
\begin{figure}[t]
    \centering

        \resizebox{.45\textwidth}{!}{
        \begin{tikzpicture}[
    latent/.style={circle, draw=black, fill=white, inner sep=1pt, minimum size=2.1em},
    node distance=.3cm and .5cm, 
    >=stealth' 
]

    \node[latent] (W_t) {$W_1$};
    \node[latent, right=of W_t] (RT_t) {$W_2$}; 
    \node[latent, right=of RT_t] (Th_t) {$W_3$}; 

    \node[above=of W_t,xshift=-.5cm,yshift=-0.3cm] (anchor1) {};
    \node[above=of H_t,xshift=-1.4cm,,yshift=-0.3cm] (anchor2) {};
    
    \draw [dashed] (anchor1) -- (anchor2);

    \node[latent, above=of W_t] (W_0) {$W_1^0$};
    \node[latent, right=of W_0] (RT_0) {$W_2^0$}; 
    \node[latent, right=of RT_0] (Th_0) {$W_3^0$}; 
 
 \node[above=of RT_0, yshift=-.3cm] {DBN-template $\mathcal{T}$};

    \node[latent, below=of W_t] (W_t1) {$W_1'$}; 
    \node[latent, right=of W_t1] (RT_t1) {$W_2'$};
    \node[latent, right=of RT_t1] (Th_t1) {$W_3'$};

    \edge {W_t} {W_t1};
          \edge {W_t1} {RT_t1};
           \edge {Th_t} {RT_t1};
      \edge {RT_t} {Th_t1};


 \node[latent, below=of W_t1, yshift=-.7cm] (a10) {$W_1^0$};
    \node[latent, right=of a10] (a20) {$W_2^0$}; 
    \node[latent, right=of a20] (a30) {$W_3^0$};
    
     \node[latent, below=of a10] (a11) {$W_1^1$};
    \node[latent, right=of a11] (a21) {$W_2^1$}; 
    \node[latent, right=of a21] (a31) {$W_3^1$};
    
    \node[latent, below=of a11] (a12) {$W_1^2$};
    \node[latent, right=of a12] (a22) {$W_2^2$}; 
    \node[latent, right=of a22] (a32) {$W_3^2$};
    
    \node[latent, below=of a12] (a13) {$W_1^3$};
    \node[latent, right=of a13] (a23) {$W_2^3$}; 
    \node[latent, right=of a23] (a33) {$W_3^3$};

  \edge {a10} {a11};
   \edge {a11} {a12};
    \edge {a12} {a13};

\edge {a11} {a21};
\edge {a12} {a22};
\edge {a13} {a23};

\edge {a20} {a31};
\edge {a21} {a32};
\edge {a22} {a33};

\edge {a30} {a21};
\edge {a31} {a22};
\edge {a32} {a23};
        
\node[above=of a20, yshift=-.2cm]  {unfolding of $\mathcal{T}$};


 \node[latent, right=of RT_0, xshift=2.7cm] (b10) {$W_1$};
    \node[latent, right=of b10] (b20) {$W_2$}; 
    \node[latent, right=of b20] (b30) {$W_3$};
    
     \node[latent, below=of b10,yshift=-.5cm] (b11) {$W_1$};
    \node[latent, right=of b11] (b21) {$W_2$}; 
    \node[latent, right=of b21] (b31) {$W_3$};
    
        \edge {b11} {b21};

     \node[latent, below=of b11,yshift=-1.7cm] (b12) {$W_1$};
    \node[latent, right=of b12] (b22) {$W_2$}; 
    \node[latent, right=of b22] (b32) {$W_3$};
     \node[latent, above=of b22] (u13) {$U_{13}$};

 \edge {b12} {b22};
  \edge {u13} {b12};
    \edge {u13} {b32};

      \node[latent, below=of b12,yshift=-1.7cm] (b13) {$W_1$};
    \node[latent, right=of b13] (b23) {$W_2$}; 
    \node[latent, right=of b23] (b33) {$W_3$};
     \node[latent, above=of b23] (u13a) {$U_{13}$};
     \node[latent, above=of b13] (u12a) {$U_{12}$};
     \node[latent, above=of b33] (u23a) {$U_{23}$};

 \edge {b13} {b23};
  \edge {u13a} {b13};
    \edge {u13a} {b33};
      \edge {u12a} {b23};
    \edge {u12a} {b13};
    \edge{u23a}{b23};
    \edge{u23a}{b33};

    \node[draw, rectangle, rounded corners, fit=(b10) (b30) ] (q0) {};
    \node[draw, rectangle, rounded corners, fit=(b11) (b31) ] (q1) {};
     \node[draw, rectangle, rounded corners, fit=(b12) (b32) (u13) ] (q2) {};
          \node[draw, rectangle, rounded corners, fit=(b13) (b33) (u13a) ] (q3) {};

    \edge {q0} {q1};
        \edge {q1} {q2};
                \edge {q2} {q3};
                      \path (q3) edge [->,loop left,looseness=3] (q3);

 \node[above=of b20, yshift=-.3cm] {transition system $\mathcal{S}$};
   
\end{tikzpicture}
}

    \caption{An example of a DBN-template $\mathcal{T}$ with its unfolding and the transition system $\mathcal{S}$ constructed from $\mathcal{T}$. Variables $U_{ij}$ without outgoing edges are omitted. Note, e.g., that the variable $U_{13}$ is connected to $W_1$ and $W_2$ after two time steps reflecting that there is a collision-free d-path $W_1^2,W_1^1,W_2^1,W_3^2$ connecting $W_1^2$ and $W_3^2$ in the unfolding through the previous time slices. }
         
    \label{fig:illustration_construction}
\end{figure}

We  construct a transition system $\mathcal{S}=(Q,q_0,\to,A,L)$, where $Q$ will be the state space with an initial state $q_0$, $\to$ a deterministic successor relation, $A$ the set of structural CI propositions, and $L\colon Q\to 2^A$ the labeling function.
The states in $S$ will be 
 BN-templates that represent the connections via collision-free d-paths in the DBN-template at some time $t$. 
 For an illustration of all steps of the construction, see Fig.~\ref{fig:illustration_construction}.

\emph{State space and labeling function:}
The representative BN-template for time $t = 0$ is simply the initial BN-template. The representative BN-templates for times $t > 0$ use variables $\mathbf{V} = \{W_1, \ldots, W_n\}$, and auxiliary variables $\mathbf{U} = \{U_{ij}: 1 \le i, j \le n\}$, corresponding to unordered pairs of distinct $i, j$. The edge $(W_i, W_j)$ is always present if and only if $(W_i', W_j')$ is an edge in the step-template BN. 
In the representative at time $t$,
we additionally draw edges $(U_{ij}, W_i)$ and $(U_{ij}, W_j)$ if there is a collision-free d-path from $W_i^t$ to $W_j^t$ in the original DBN-template, and the intermediate vertices along this path do not belong to $\mathbf{V}^t$. Note how the latter requirement rules out the possibility of a path being blocked by an observed variable. 
So, by definition, $(\mathbf{X} \ind \mathbf{Y} \mid \mathbf{Z})$ holds at time $t$ in the DBN if and only if $(\mathbf{X}^t \ind \mathbf{Y}^t \mid \mathbf{Z}^t)$ holds in the representative BN-template; running d-separation queries on the latter settles the issue of the labelling function, which hence can be computed in time polynomial in $|\mathbf{V}|$ and $A$.

We observe that for $t > 0$, there are $2^{\binom{|\mathbf{V}|}{2}}$ possible representatives, depending on the choice of which $U_{ij}$ are made parents, each representable as a graph with $O(|\mathbf{V}|^2)$ vertices. This establishes the size requirements described at the beginning of Sec.~\ref{section:structural}.

\emph{Transition relation:}
Above, the transition relation of the reachable part of the state space is implicitly given by following the time steps.
Now, we describe how to compute the successor of a representative BN-template  $\mathcal{B}$ on $\mathbf{V}\cup\mathbf{U}$ directly in polynomial time.
For this, consider the graph with vertices $\mathbf{U} \cup \mathbf{V} \cup \mathbf{V}'\cup\mathbf{U}'$. Draw edges in the subgraph induced by $\mathbf{U} \cup \mathbf{V}$ as prescribed by  $\mathcal{B}$, and edges in $(\mathbf{V} \cup \mathbf{V}') \times (\mathbf{V}')$ as prescribed by the step template. 
Now, for each $i\not=j$, we add edges from $U_{ij}'$ to $W_i'$ and $W_j'$ if there is an $A\in \mathbf{U} \cup \mathbf{V} $ that can reach $W_i'$ and $W_j'$ in the graph constructed so far without using edges inside $\mathbf{V}'$. 

The correctness of this construction can be seen as follows:
If $A$ is in $\mathbf{V}$, this is a collision-free d-path from $W_i'$ to $W_j'$. If $A\in \mathbf{U}$ it has edges to some $B_i\in\mathbf{V}$ reaching $W_i'$ and $B_j\in \mathbf{V}$ reaching $W_j'$. The edges to $B_i$ and $B_j$ 
represent a collision-free d-path which can then be extended to a collision-free d-path from $W_i'$ to $W_j'$.
Finally, we restrict the graph to $\mathbf{V}'\cup\mathbf{U}'$ to obtain the successor of $\mathcal{B}$.
Clearly, this successor can be computed in time polynomial in $|\mathbf{V}|$.

\emph{PSPACE upper bound:}
So, we have constructed the deterministic transition system satisfying the requirements described at the beginning of Sec.~\ref{section:structural} and conclude:

\begin{theorem}
\label{theorem-pspace}
The structural LTL and NBA model-checking problems for DBN-templates are in PSPACE.
\end{theorem}

\subsection{The special case of restricted DBNs}
The difficulty of exponentially long periods is circumvented when we consider restricted DBNs: in this case, the trace is constant from position $t = |\mathbf{V}|^2$ onwards. 
To show this, we argue that for restricted DBNs, all times $t > |\mathbf{V}|^2$ have the same representative. Since going all the way to $\mathbf{V}^0$ does not give access to any ``new'' connecting edges in this setting. Formally:

\begin{lemma}
In a restricted DBN-template, if there is a collision-free path from $X^t$ to $Y^t$, there is one that goes back at most $|\mathbf{V}|^2$ time slices.
\end{lemma}
\begin{proof}
Suppose, for the sake of contradiction, a path with the lowest ``traceback'' goes back more than $|\mathbf{V}|^2$ time slices. Then, by the pigeonhole principle, there exist $i < j$ such that the path enters and exits $\mathbf{V}^{t-i}$ and $\mathbf{V}^{t-j}$ at the same pair of variables. We can thus replace the path from $\mathbf{V}^{t-i}$ with the path from $\mathbf{V}^{t-j}$ and get a collision-free path with a smaller traceback: a contradiction, as desired.
\end{proof}

Thus, in the case of restricted DBNs, the trace is ultimately constant after at most $|\mathbf{V}|^2$ time steps, and the entire transition system can be computed in time polynomial in $|\mathbf{V}|, |A|$. This reduced complexity allows us to use standard techniques to show:

\begin{theorem}
\label{lemma-restricted-poly}
The LTL and NBA model-checking problems for restricted DBNs with structural-independence propositions can be solved in polynomial time.
\end{theorem}
\begin{proof}
The case of NBA model checking is immediate. We can construct an automaton whose language is precisely the trace $\tau$, intersect it with the given NBA, and check the result for non-emptiness, all in polynomial time \cite[Chapter 4.3]{baierkatoen}.

To check whether an ultimately constant trace satisfies an LTL formula $\varphi$, we adopt a dynamic programming approach, memoizing for $0 \le t \le |\mathbf{V}|^2+1 = T$ whether a suffix $\tau[t:\infty]$ satisfies a subformula $\varphi'$ of $\varphi$. We populate entries from larger to smaller $t$, and simpler to more complex $\varphi'$. Atomic propositions and Boolean connectives are handled in the obvious way. The suffix from time $t = T$ satisfies $\LTLnext \varphi'$ if and only if it satisfies $\varphi'$; it satisfies $\varphi_1 \LTLuntil \varphi_2$ if and only if it satisfies $\varphi_2$. For smaller $t$, the suffix from time $t$ satisfies $\LTLnext \varphi'$ if and only if the suffix from time $t+1$ satisfies $\varphi$; it satisfies $\varphi_1 \LTLuntil \varphi_2$ if and only if it either satisfies $\varphi_2$, or it satisfies $\varphi_1$ and the suffix from $t+1$ satisfies $\varphi_1 \LTLuntil \varphi_2$. Having populated the table, we check whether the suffix from $t=0$ satisfies the given $\varphi$.
\end{proof}

\section{Stochastic Conditional Independence}
\label{section:skolem}

In this section, we  establish the number-theoretic hardness of reasoning about stochastic CIs when concrete conditional probability distributions are given. Specifically, we shall show that deciding whether formulae of the form $\LTLeventually (X \probind Y)$ hold is at least as hard as the Skolem problem for rational linear recurrence sequences (LRS).

A rational LRS of order $k$ is a sequence $(u_n)_{n=0}^\infty$ of rational numbers satisfying the recurrence relation $u_{n+k} = a_{k-1}u_{n+k-1} + \cdots + a_0 u_n$, where $a_{k-1}, \ldots, a_0$ are rational numbers with $a_0 \ne 0$. It is given by the coefficients $a_0, \ldots, a_{k-1}$, and the initial terms $u_0, \ldots, u_{k-1}$. The Skolem problem takes as input an LRS, where the recurrence relation and initial terms are respectively encoded as a companion matrix $A \in \mathbb{Q}^{k\times k}$ and a vector $u \in \mathbb{Q}^k$, and asks whether there exists an $n$ such that $u_n = 0$, i.e., $A^n u$ contains a $0$-entry. 

The Skolem problem has been open for nearly a century \cite{everestmonograph, tao2008}. It is open even if we restrict the LRS to have order five \cite{ouaknineworrellrpsurvey}. It is known to be decidable at orders four and below \cite{Tijdeman1984, vereshchagin}.

\begin{lemma}
\label{lem-skolem-hardness}
Consider a rational LRS of order $k$, given by its companion matrix $A \in \mathbb{Q}^{k \times k}$, and vector $u \in \mathbb{Q}^k$ of initial values. We can compute a DBN with $\lceil \log k \rceil + 2$ binary variables $X, Z_1, \ldots, Z_\ell, Y$ where $\ell=\lceil \log k \rceil$ and rational conditional probabilities, such that $\LTLeventually(X \probind Y)$ holds if and only if the LRS has a zero term.
\end{lemma}

The reduction uses \cite[Cor.~1]{Aghamov2025} to ``embed'' the given LRS into a Markov chain $(M, v)$, and then Lem.~\ref{lemma:unfolding} to convert the Markov chain into a DBN. We remark that as a corollary, our construction can also be used to reduce the closely related Positivity problem for LRS (see, e.g., \cite{soda2014positivity} for arguments of number-theoretic hardness) to the problem of deciding whether $Y$ always ``positively influences'' $X$.
\begin{proof}
We construct a DBN such that at the $n$-th time step:
\begin{itemize}
\item The event $Y = 1$ occurs unconditionally with probability $1/2$.
\item The difference $\Pr[X = 1 | Y =1 ] - \Pr[X=1 | Y = 0]$ is $2^{-n} \rho^n \eta$ times the $n$-th term of the LRS, where $\eta, \rho$ are positive rational constants.
\end{itemize}
At a high level, the construction proceeds as follows:
\begin{enumerate}
\item We use \cite[Cor.~1]{Aghamov2025} to ``embed'' the given instance $(A, u)$ of order $k$ into an ergodic Markov chain $(M, v)$ of order $k+1$.
\item We encode the states of the Markov chain with binary variables, or ``bits'' $X, Z_0, \ldots, Z_\ell$, where $X$ indicates whether the system is in the first state, and all other states are indicated by the usual binary encoding introduced in Lem.~\ref{lemma:unfolding}. This ensures that using even restricted DBNs suffices for the reduction.
\item In the DBN, the current values of $X, Z_0, \ldots, Z_\ell$ depend on not only on the previous values, but also on the current value of $Y$. If $Y = 1$, the distribution of current values is obtained from the previous values as per the construction in Lem.~\ref{lemma:unfolding}. Otherwise, the distribution is ``fast-forwarded'' to the stationary distribution $s$ of $M$.
\end{enumerate}

Let $s \in \mathbb{Q}^{k+1}$ be the uniform distribution, and $S$ be the square matrix whose columns are all $s$. By \cite[Cor.~1]{Aghamov2025}, we can compute an ergodic Markov chain $M$ and an initial distribution $v$ (both with all entries rational) such that for all $n$, 
\begin{equation}
M^n v = s + \eta \rho^n \begin{bmatrix} I \\ -\mathbf{1}_k^\top \end{bmatrix}A^n u,
\label{embedding}
\end{equation}
where $\eta, \rho$ are positive rational constants, and $\mathbf{1}_k$ denotes the vector with all entries equal to $1$.

We label the states of the Markov chain as $\alpha, 0, 1, \ldots, k-1$, and encode them with bits $X, Z_0, \ldots, Z_\ell$. The encoding of $\alpha$ has $X = 1$ and all other bits $0$, the encoding of any other state $i$ has $X = 0$ and other bits set according to the binary representation of $i$. Clearly, this uses $\lceil \log k \rceil + 1$ bits.

We shall also index the rows and columns of $M$ by $\alpha, 0, 1, \ldots, k-1$, such that $\alpha$ corresponds to the topmost, and leftmost.

To construct the DBN including the additional bit $c$, we replicate the construction of Lem.~\ref{lemma:unfolding}. The initial Bayesian network is set up so that valid state encodings with $Y = 0$ each get half the probability prescribed by $s$, and valid state encodings with $Y = 1$ each get half the probability prescribed by $v$.

For the step-template, we have that the current value of $Y$ is unconditionally assigned uniformly at random. If $Y = 1$, the conditional distribution of bits $X, Z_0, \ldots, Z_\ell$ is the same as that prescribed by $M$, i.e., if the previous values encoded a valid state $i$, then the conditional probability of valid encoding $j$ is the $(j, i)$-th entry of $M$; if the previous value was an invalid encoding $i$, then the current value is deterministically $i$. Similarly, if $Y = 0$: if the previous values encoded a valid state $i$, then all valid encodings $j$ are assigned probability $1/(k+1)$, if the previous encoding $i$ was invalid, then the current encoding is deterministically $i$.

We note that only the valid encodings are reachable, and $\frac{1}{2}\begin{bmatrix}M & M \\ S & S\end{bmatrix}$ is an equivalent Markov chain for (the reachable configurations of) the DBN, and the initial distribution is $\frac{1}{2}\begin{bmatrix}v \\ s\end{bmatrix}$.

Intuitively, at step $n$, the probability that $X, Z_1, \ldots, Z_\ell$ encode state $i$ is the probability that there was never a fast-forward (which is $2^{-n-1}$) times what the probability would be according to the Markov chain (which is $e_i^\top M^n v$), plus the probability there was a fast-forward times the stationary probability (which is $e_i^\top s$), where $e_i$ is the vector whose entry at index $i$ is $1$ and other entries are $0$.

Formally, at step $n$,

$
\Pr[X=1] = \frac{1}{2^{n+1}} e_\alpha^\top M^n v + \frac{2^{n+1} - 1}{2^{n+1}} \cdot \frac{1}{k+1}.
$

We can also check this via the equivalent Markov chain of the DBN. We can show via a simple induction, and using the facts that $MS = SM = SS = S$, that 
$$
\begin{bmatrix}M & M \\ S & S\end{bmatrix}^n = 
\begin{bmatrix}
M^n + (2^{n-1} - 1)S & M^n + (2^{n-1} - 1)S \\
2^{n-1}S & 2^{n-1}S
\end{bmatrix}
$$
Using the fact that $Sv = s$, we have that at step $n$, the distribution is $\frac{1}{2^{n+1}} \begin{bmatrix}M^n v + (2^n-1)s \\ 2^n s\end{bmatrix}$.

Observe that
$
\Pr[X=1 \mid Y=1] = \frac{1}{2^{n}} e_\alpha^\top M^n v + \frac{2^{n} - 1}{2^{n}} \cdot \frac{1}{k+1},
$
and that $\Pr[X=1 \mid Y=0] = \frac{1}{k+1}$, and their difference is $\frac{1}{2^n} \left (e_\alpha^\top M^n v - \frac{1}{k+1}\right)$. Recall that by the embedding \ref{embedding} of the LRS by \cite[Cor.~1]{Aghamov2025}, this can be rewritten as $2^{-n} \eta \rho^n (e^\top A^n u)$, i.e., a scaled version of the LRS. We have thus shown that at time $n$,
$
\Pr[X=1 \mid Y=1] - \Pr[X=1  \mid Y=0] = 2^{-n} \eta \rho^n u(n),
$
where $u(n)$ is the $n$-th term of the LRS, and $\eta, \rho$ are rational constants. The reductions of Skolem to eventual independence of $X$ and $Y$ and Positivity to global causation of $X$ by $Y$ is complete.
\end{proof}

\section{Discussion: Faithfulness in DBNs}
\label{section:discussion}

Prop.~\ref{sound&complete} demonstrates an analog of Thm.~\ref{thm:BN-equivalence} for DBNs, albeit in one direction. However, in future work, we aim to formally prove that the concept of structural independence is faithful to stochastic independence in DBNs, establishing a complete analog of Thm.~\ref{thm:BN-equivalence} for DBNs. The distinction from the known result is that, when transitioning to DBNs, we impose constraints on the parameters by identifying the distributions of the same variables across different time slices. This reduces the dimensionality of the parameter space, leading to a strictly smaller family of admissible Bayesian networks at any given time t.

We say that the parameters $\langle \mathcal{P}^0, \mathcal{P}^\mathsf{step} \rangle$ are \emph{$t$-unfaithful} if
$(\mathbf{X} \probind \mathbf{Y} \mid \mathbf{Z}) \quad \text{holds at time } t \quad \text{but} \quad (\mathbf{X} \ind \mathbf{Y} \mid \mathbf{Z}) \quad \text{does not}$.
They are called \emph{unfaithful} if this occurs for some $t$. In other words, unfaithful parameters are those for which the structural and stochastic conditional independencies diverge. We aim to prove that only a measure-$0$ set of parameters is unfaithful.

Recall that by definition, if $(\mathbf{X}^t \probind \mathbf{Y}^t \mid \mathbf{Z}^t)$, then for every $\mathbf{x}, \mathbf{y}, \mathbf{z}$, we must have that the expression
\begin{align*}
&\Pr[(\mathbf{X}^t, \mathbf{Y}^t, \mathbf{Z}^t) = (\mathbf{x}, \mathbf{y}, \mathbf{z})]\cdot \Pr[\mathbf{Z}^t = \mathbf{z}]  \\ - &\Pr[(\mathbf{X}^t, \mathbf{Z}^t) = (\mathbf{x}, \mathbf{z})] \cdot \Pr[(\mathbf{Y}^t, \mathbf{Z}^t) = (\mathbf{y}, \mathbf{z})]
\end{align*}
is equal to $0$. We observe that in the DBN setting, we can use Lem.~\ref{lemma:unfolding} to argue that the above expression is a (degree $O(t)$) polynomial in the parameters $\langle\mathcal{P}^0, \mathcal{P}^\mathsf{step}\rangle$. In particular, if the polynomial is not identically $0$, then for all but a measure-$0$ set of parameters, it returns a nonzero value. Our strategy is thus to prove that if structural dependence holds, the corresponding polynomial cannot be identically $0$. Using that zero-sets of polynomials have measure $0$ and are closed under countable unions, we would deduce that unfaithful parameters form a measure-$0$ set.

We remark that the proof of Thm.~\ref{thm:BN-equivalence} in \cite[Section~6.4]{Meek} relies on a \emph{local dependence} condition: if $(X, Y)$ is an edge, then $X$ and $Y$ must be dependent. In Bayesian networks, CPDs can be chosen so that variables along a $d$-path are locally dependent, while all others remain independent of any other variables. This is not possible in dynamic Bayesian networks, where temporal and structural constraints prevent such isolation. This limitation is a key obstacle to extending the argument to the DBN setting. A potential approach is to consider a CPD where a variable takes value $1$ with higher probability whenever the majority of its parents are $1$. While we focus on binary variables here, establishing the result in this setting would yield the general case as a straightforward corollary.  Another promising direction comes from algebraic statistics \cite{alstat2018}, which applies tools from algebraic geometry and combinatorics to study statistical models, especially those involving discrete data. In Bayesian networks, it encodes conditional independencies as polynomial equations and analyzes the resulting algebraic varieties to understand structural and probabilistic properties. Another line of attack could also possibly involve observing that the sequence of polynomials characterizing conditional (in)dependence at time $t$ forms a linear recurrence over the field of multivariate rational functions, and judiciously appealing to the Skolem-Mahler-Lech theorem (the set of zeroes of a linear recurrence over a field of characteristic $0$ is the union of a finite set and finitely many effective arithmetic progressions).

\section{Conclusion}
We introduced LTL-based and NBA-based specification formalisms to express temporal properties of CIs in DBNs.
These formalisms can express desirable system properties such as non-interference in security applications
and open the possibility to verify systems against all kinds of desirable specifications regarding the temporal evolution of CIs.

We restricted here to CI propositions that state CIs between variables at the same time slice.
Our techniques, however, offer the possibility to 
  introduce CI propositions talking about variables at different time slices. A syntax for such propositions  could, e.g., be $(\mathbf{X}^{+2} \ind \mathbf{Y} \mid \mathbf{Z}^{+1})$,
 which holds at time point $t$ in a DBN-template if $(\mathbf{X}^{t+2} \ind \mathbf{Y}^t \mid \mathbf{Z}^{t+1})$.
 If the entries expressing the time shifts are bounded by some $k$ given in unary, out  model-checking algorithm for DBN-templates can be adapted without increasing the asymptotic complexity:
We can adapt the construction of the deterministic transition system encoding the trace of the DBN-template by letting the states
consist of BN-representatives that unfold the step template for $k$ steps before adding  the ``tunnelling through''-layer of variables encoding the existence of collision-free d-paths between variables.

Regarding stochastic CIs in DBNs, our Skolem-hardness 
result is sobering regarding the potential of verifying systems against temporal specifications with respect to stochastic CI statements---which might come as a surprise. 
The key  to obtain this hardness result was establishing the intricate connection between LRSs and DBNs.

\section*{Acknowledgments}
The authors were supported  by the DFG through the DFG grant 389792660 as part of TRR 248 (Foundations of Perspicuous Software Systems, see https://www.perspicuous-computing.science) and the Cluster of
Excellence EXC 2050/1 (CeTI, project ID 390696704, as
part of Germany’s Excellence Strategy) and by the BMBF (Federal Ministry of Education and Research) in DAAD project
57616814 (SECAI, School of Embedded and Composite AI)
as part of the program Konrad Zuse Schools of Excellence
in Artificial Intelligence. Jo\"el Ouaknine is also aﬃliated with Keble College, Oxford, as emmy.network Fellow.

\bibliography{references-final}

\end{document}